\documentclass[letterpaper]{article} 
\usepackage{aaai25}  
\usepackage{times}  
\usepackage{helvet}  
\usepackage{courier}  
\usepackage[hyphens]{url}  
\usepackage{graphicx} 
\usepackage{natbib}  
\usepackage{caption} 
\usepackage{algorithm}
\usepackage{algorithmic}
\usepackage{amsthm,amsmath,amssymb,amsfonts}
\usepackage{multirow}
\usepackage{booktabs}
\usepackage{makecell} 
\usepackage{diagbox}
\usepackage{subcaption}
\usepackage{xcolor}

\usepackage{newfloat}
\usepackage{listings}
\DeclareCaptionStyle{ruled}{labelfont=normalfont,labelsep=colon,strut=off} 
\floatstyle{ruled}
\newfloat{listing}{tb}{lst}{}
\floatname{listing}{Listing}

\newcommand{\ourmethod}{SFR-GNN}

\newtheorem{theorem}{Theorem}
\newtheorem{lemma}{Lemma}

\title{\ourmethod: Simple and Fast Robust GNNs against Structural Attacks}
\author {
    Xing Ai\textsuperscript{\rm 1},
    Guanyu Zhu\textsuperscript{\rm 1},
    Yulin Zhu\textsuperscript{\rm 1},
    Yu Zheng\textsuperscript{\rm 1 \rm 2},
    Gaolei Li\textsuperscript{\rm 3},
    Jianhua Li\textsuperscript{\rm 3},
    Kai Zhou\textsuperscript{\rm 1}
}
\affiliations {
    \textsuperscript{\rm 1}The Hong Kong Polytechnic University, HKSAR\\
    \textsuperscript{\rm 2}The Chinese University of Hong Kong, HKSAR\\
    \textsuperscript{\rm 3}Shanghai Jiao Tong University, Shanghai, China\\
}

\usepackage{bibentry}

\begin{document}

\maketitle

\begin{abstract}
Graph Neural Networks (GNNs) have demonstrated commendable performance for graph-structured data. Yet, GNNs are often vulnerable to adversarial structural attacks as embedding generation relies on graph topology. Existing efforts are dedicated to purifying the maliciously modified structure or applying adaptive aggregation, thereby enhancing the robustness against adversarial structural attacks. It is inevitable for a defender to consume heavy computational costs due to lacking prior knowledge about modified structures. To this end, we propose an efficient defense method, called {\underline S}imple and {\underline F}ast {\underline R}obust {\underline G}raph {\underline N}eural {\underline N}etwork (\ourmethod), supported by mutual information theory. The \ourmethod\  first pre-trains a GNN model using node attributes and then fine-tunes it over the modified graph in the manner of contrastive learning, which is free of purifying modified structures and adaptive aggregation, thus achieving great efficiency gains. Consequently, \ourmethod\  exhibits a 24\%--162\% speedup compared to advanced robust models, demonstrating superior robustness for node classification tasks. 
\end{abstract}

%

\section{Introduction}

Graph Neural Networks (GNNs) have emerged as the leading approach for graph learning tasks across various domains, including recommender system~\cite{zhang2024hi}, social networks \cite{hu2023cost}, and bioinformatics \cite{NEURIPS2023_4db8a681}.
However, numerous studies \cite{Nettack,Minmax,hussain2021structack,zhu2022binarizedattack} have demonstrated the vulnerability of GNNs under \textit{adversarial attacks}, where an attacker can deliberately modify the graph data to cause the misprediction of GNNs. Among them, \textit{structural attacks} \cite{Nettack,Mettack,GraD} have gained prominence due to the unique structural nature of graph data. Specifically, by solely modifying the edges in a graph, structural attacks
hold practical significance in application scenarios where attackers have limited access to the relationships among entities rather than the attributes of the entities themselves.

To defend against structural attacks, numerous robust GNN models \cite{Pro-GNN,GCN-Jaccard,SimP-GCN,FocusedCleaner} have been proposed recently. The main ideas behind these approaches involve purifying the modified graph structure or designing adaptive aggregation mechanisms to avoid aggregating messages through poisoned structures. 
Despite these efforts, existing robust GNNs still encounter significant scalability challenges, which hinder their application in practical scenarios. These scalability issues are mainly attributed to two factors: \textit{computational complexity} and \textit{hyper-parameter complexity}.

\begin{figure}[!t]
    \centering
    \includegraphics[scale=0.55]{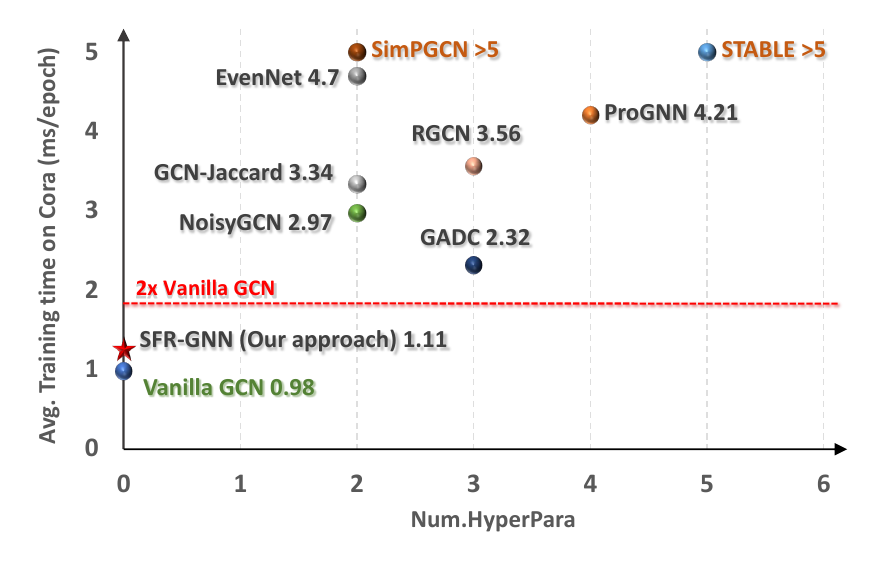}
    \caption{Computational Complexity and Hyper\-parameter Complexity Comparison of Existing Robust GNNs. Our method \ourmethod\ is highlighted with a red star. }
    \label{fig:complexity}
\end{figure}

Specifically, recent research \cite{zhu2021deep, EvenNet, NoisyGCN} reveals that current robust GNN models suffer from high computational complexity due to complex defense mechanisms, such as learning new graph structures or computing edge attention. 
Moreover, these robust models often introduce additional hyper-parameters (e.g. weighting coefficients and thresholds) beyond the basic ones (e.g. learning rates, dropout ratio, epochs).
Unfortunately, effective hyper-parameter tuning often requires extensive background knowledge of the data, which may not always be available due to issues like distribution shifts. The interactions among multiple hyper-parameters compel developers to employ techniques such as grid search or cross-validation to ensure optimal values for all hyper-parameters, complicating model deployment in real-world scenarios \cite{wang2021bag,chen2022bag}. 

Fig.~\ref{fig:complexity} compares state-of-the-art robust GNNs with the vanilla GCN in terms of training time (per epoch in milliseconds) on Cora dataset and the number of extra hyper-parameters. The results indicate that all robust GNNs require more than twice the runtime of the vanilla GCN and introduce at least two extra hyper-parameters (with our method as an exception). It demonstrates that while existing robust GNNs achieve adversarial robustness, it comes at a significant cost in training time and hyper-parameter complexity. 


In response, a natural question emerges: \textit{can we develop a GNN model that achieves adversarial robustness against structural attacks while also being simple and fast?}

In this work, we propose a {\underline S}imple and {\underline F}ast {\underline R}obust {\underline G}raph {\underline N}eural {\underline N}etwork (\ourmethod) that employs a simple but effective learning strategy: pre-training on node attributes and then fine-tuning on structure information. Specifically, given a positioned graph $\mathcal{G}'=(\mathbf{X},\mathbf{A}', \mathbf{Y})$ with manipulated structure $\mathbf{A}'$,  
\ourmethod\ is initially pre-trained using only node attributes $\mathbf{X}$ without structural information involved. Subsequently, the model is fine-tuned over $\mathbf{A}'$, devising a \textit{specialized} graph contrastive learning scheme.

The idea behind this strategy is rooted in the analysis of the structural attack: the attacker meticulously generates a modified structure $\mathbf{A}'$ based on the given node attributes $\mathbf{X}$, which is detrimental to the performance of GNN \textit{with respect to the corresponding $\mathbf{X}$}. Theoretically, structural attacks contaminate the mutual information between $\mathbf{A}'$ and $\mathbf{Y}$ conditioned by $\mathbf{X}$: $I(\mathbf{A}';\mathbf{Y}|\mathbf{X})$ to mislead GNN predictions. However, we indicate the ``\textbf{paired effect}" of structural attacks: $\mathbf{A}'$ is most effective alongside the given $\mathbf{X}$, and is not quite effective with any other $\mathbf{X}'\neq \mathbf{X}$ (detailed in Sec.~\ref{sec:proof}).
Therefore, our strategy achieves robustness against structural attacks by disrupting the ``\textbf{paired effect}". This is achieved by first pre-training on attributes $\mathbf{X}$ to obtain node embedding $\mathbf{Z}_{p}$ and then fine-tunes it paired with $\mathbf{A}'$ to incorporate structural information, which actually pairs $\mathbf{A}'$ with $\mathbf{Z}_{p}$ instead of $\mathbf{X}$, thus mitigating the impact of the manipulated structure on model performance. Despite its simplicity, we provide theoretical support and insights through a mutual information perspective in Sec.~\ref{sec:proof}.



As a result, \ourmethod\ features a lightweight construction with no \textit{additional} hyper-parameters, significantly alleviating the computational and hyper-parameter complexity associated with building robust GNNs. Fig.~\ref{fig:complexity} illustrates that the computational and hyperparameter complexity of \ourmethod\ is close to that of vanilla GCN and outperforms existing robust GNNs, highlighting the simplicity and ease of implementation of \ourmethod. Datasets and codes of this paper are at the supplements.

Our major contributions are summarized as follows.

\begin{itemize}
    \item[1)] We propose a novel, simple and fast robust GNN model named \ourmethod\  that employs an ``attribute pre-training and structure fine-tuning" strategy to achieve robustness against structural attacks. This approach is efficient in that it requires no extra hyper-parameters and is free of time-consuming operations such as purification or attention mechanisms.
    
    \item[2)] We offer a comprehensive theoretical analysis through mutual information theory to provide insights into the proposed method and substantiate its effectiveness. 
    
    \item[3)] The comprehensive evaluation of \ourmethod\ against state-of-the-art baselines on node classification benchmarks, encompassing large-scale graph datasets, reveals that it achieves comparable or superior robustness while significantly enhancing runtime efficiency by a range of 24\% to 162\% compared to state-of-the-art baselines.
    
\end{itemize}

\subsection{Related Works} 

\paragraph{Structural Attacks in Graph Learning.} 
Structural attacks are a popular form of attack covering a wide range of graph learning models beyond GNNs, including
self-supervised learning \cite{bojchevski2019adversarial,zhang2022unsupervised}, signed graph analysis \cite{zhou2023black,zhu2024towards}, recommender systems \cite{lai2023towards}, and so on. The primary idea is to utilize gradient-based methods to search for the optimal graph structure to degrade the performances of various tasks. For instance, Mettack~\cite{Mettack} formulated the global structural poisoning attacks on GNNs as a bi-level optimization problem and leveraged a meta-learning framework to solve it. BinarizedAttack~\cite{BinarizedAttack} simplified graph poisoning attacks against the graph-based anomaly detection to a one-level optimization problem.
HRAT~\cite{HRAT} proposed a heuristic optimization model integrated with reinforcement learning to optimize the structural attacks against Android malware detection. GraD \cite{GraD} 
proposes a reasonable budget allocation mechanism to enhance the effects of structural attacks.

\paragraph{Robust GNNs.}
To defend against structural attacks, a series of robust GNNs are proposed, which typically rely on purifying the modified structure or designing adaptive aggregation strategies. For example, 
GNNGUARD~\cite{GNNGUARD} removes the malicious links during training by considering the cosine similarity of node attributes. 
Zhao et al. \cite{hang-quad} used a conservative Hamiltonian flow to improve the model's performance under adversarial attacks. 
However, common drawbacks of these approaches include high computational overhead and hyper-parameter complexity.
More recently, few works have attempted to develop efficient robust GNNs. For example, NoisyGCN \cite{NoisyGCN} defends against structural attacks by injecting random noise into the architecture of GCN, thereby avoiding complex strategies and improving runtime. Similarly, EvenNet \cite{EvenNet} proposes an efficient strategy that ignores odd-hop neighbors of nodes, with a time complexity that is linear to the number of nodes and edges in the input graph. These efforts significantly reduce the time complexity of building robust GNNs but still introduce additional hyper-parameters. NoisyGCN requires careful selection of the ratio of injected noise
and EvenNet requires the determination of both the order of the graph filter $K$ and the initialization hyper-parameter $\alpha$. This motivates us to develop even simpler while robust GNNs.

\section{Background}
We consider the node classification task in a semi-supervised setting. Specifically, let $\mathcal{G}=(\mathbf{X},\mathbf{A}, \mathbf{Y})$ be an input graph, where $\mathbf{X}\in\mathbb{R}^{n\times d}$ denotes node attributes, $\mathbf{A}\in \{0,1\}^{n\times n}$ is the adjacent matrix, and $\mathbf{Y}$ represents the partially available labels of the nodes in the training set. A GNN model $f_\theta$ parameterized with $\theta$ is trained to predict the remaining node labels through minimizing the training loss $\mathcal{L}_{tr}$ given the training node labels:
\begin{equation}
    \theta^* = \underset{\theta} {\operatorname{arg\,min}}\ \mathcal{L}_{tr}(f_{\theta}(\mathbf{X},\mathbf{A}), \mathbf{Y}).
\end{equation}
This training loss $\mathcal{L}_{tr}$ is commonly employed classification loss such as the Negative Log-Likelihood.

\paragraph{Structural Attacks.}
Structural attacks against semi-supervised node classification naturally fit within a \textit{poisoning attack setting}, where the GNN model is trained and makes predictions over a manipulated graph. In a worst-case scenario, it is assumed that the attacker can arbitrarily modify the graph structure (i.e., $\mathbf{A}$) with the goal of degrading classification performance. Specifically, the attacker seeks to find an optimal structural perturbation $\delta^*$, resulting in a poisoned graph $\mathcal{G}' = (\mathbf{X}, \mathbf{A}' = \mathbf{A} + \delta^*, \mathbf{Y})$. Mathematically, a structural attack can be formulated as solving a bi-level optimization problem:
\begin{align}
    \delta^{*} & = \underset{\delta}{\operatorname{arg\,min}}\ \mathcal{L}_{atk}(f_{\theta^*}(\mathbf{X},(\mathbf{A}+\delta)),\mathbf{Y}) \nonumber \\
     &\text{s.t.} \quad  \theta^* = \underset{\theta} {\operatorname{arg\,min}}\ \mathcal{L}_{tr}(f_{\theta}(\mathbf{X},\mathbf{A} + \delta), \mathbf{Y}),
\end{align}
where $\mathcal{L}_{atk}$ quantifies the attack objective. The attacks \cite{Mettack,GraD,Minmax} mainly differ in their specific algorithms to solve the optimization problem. 



\paragraph{Robust GNNs as Defense.}
Training robust GNN models is a common defense strategy to mitigate structural attacks. In this paper, the defender's goal is to train a robust GNN model from  the poisoned graph to maintain node classification accuracy. We note that the defender only has access to the poisoned graph $\mathcal{G}' = (\mathbf{X}, \mathbf{A} + \delta^*, \mathbf{Y})$, not the clean graph $\mathcal{G}$. Additionally, the defender does not have prior knowledge about how the perturbation $\delta^*$ was generated.



\section{Methodology}
\label{sec::method}
\subsection{Design Intuition}
\label{sec:intuition}

We propose a novel framework for efficiently learning robust GNNs against structural attacks, which employs a straightforward strategy: \textit{attribute pre-training and structure fine-tuning}, to alleviate computational and hyper-parameter complexity. We articulate this design and the intuition behind it through an information theoretical perspective. 
For completeness and better readability, we defer all theoretical analysis to Section~\ref{sec:proof}. 

Our intuition starts from a key observation of the essence of attacks: \textit{structural attacks degrade GNN's performance by contaminating the mutual information between $\mathbf{A}$ and $\mathbf{Y}$} \textbf{conditioned on $\mathbf{X}$}, denoted as $I(\mathbf{A};\mathbf{Y}|\mathbf{X})$ (see Lemma \ref{lemma:1} for details). That is, given fixed attributes $\mathbf{X}$, the attacker can generate a poisoned structure $\mathbf{A}'$ to effectively attack GNNs.
Moreover, the structural attack has a ``\textbf{paired effect}": the poisoned structure $\mathbf{A}'$ is effective with the given $\mathbf{X}$, and is not quite effective with any other $\mathbf{X}'\neq \mathbf{X}$ (see Lemma.~\ref{lemma:2}).

The above analysis reveals the key to designing robust GNNs: create a mismatch between $\mathbf{A}'$ and the attributes $\mathbf{X}$. Previous works did so by trying to purify $\mathbf{A}'$, however, with high computational complexity. 
We employ a totally different strategy: attribute pre-training and structure fine-tuning essentially involves obtaining a latent node embedding $\mathbf{Z}$ through pre-train on $\mathbf{X}$, and then fine-tune $\mathbf{Z}$ with $\mathbf{A}'$ to incorporate structural information. This approach allows the model to learn from the \textit{less harmful} $I(\mathbf{A}';\mathbf{Y}|\mathbf{Z})$ instead of the contaminated $I(\mathbf{A}';\mathbf{Y}|\mathbf{X})$ (see Theorem.~\ref{theorem:1}).

However, since $\mathbf{Z}$ is pre-trained from $\mathbf{X}$, 
there exists overlap between $I(\mathbf{A}';\mathbf{Y}|\mathbf{Z})$ and $I(\mathbf{A}';\mathbf{Y}|\mathbf{X})$, meaning that structural attacks still affects $I(\mathbf{A}';\mathbf{Y}|\mathbf{Z})$. We thus further propose a novel contrastive learning approach to learn structural information from $I(\mathbf{A}';\mathbf{Y}|\mathbf{Z})$ while mitigating the attack effect (see Theorem.~\ref{theorem:2}). The detailed constructions are presented in the next section.

\begin{figure}[!t]
    \centering
    \includegraphics[scale=0.26]{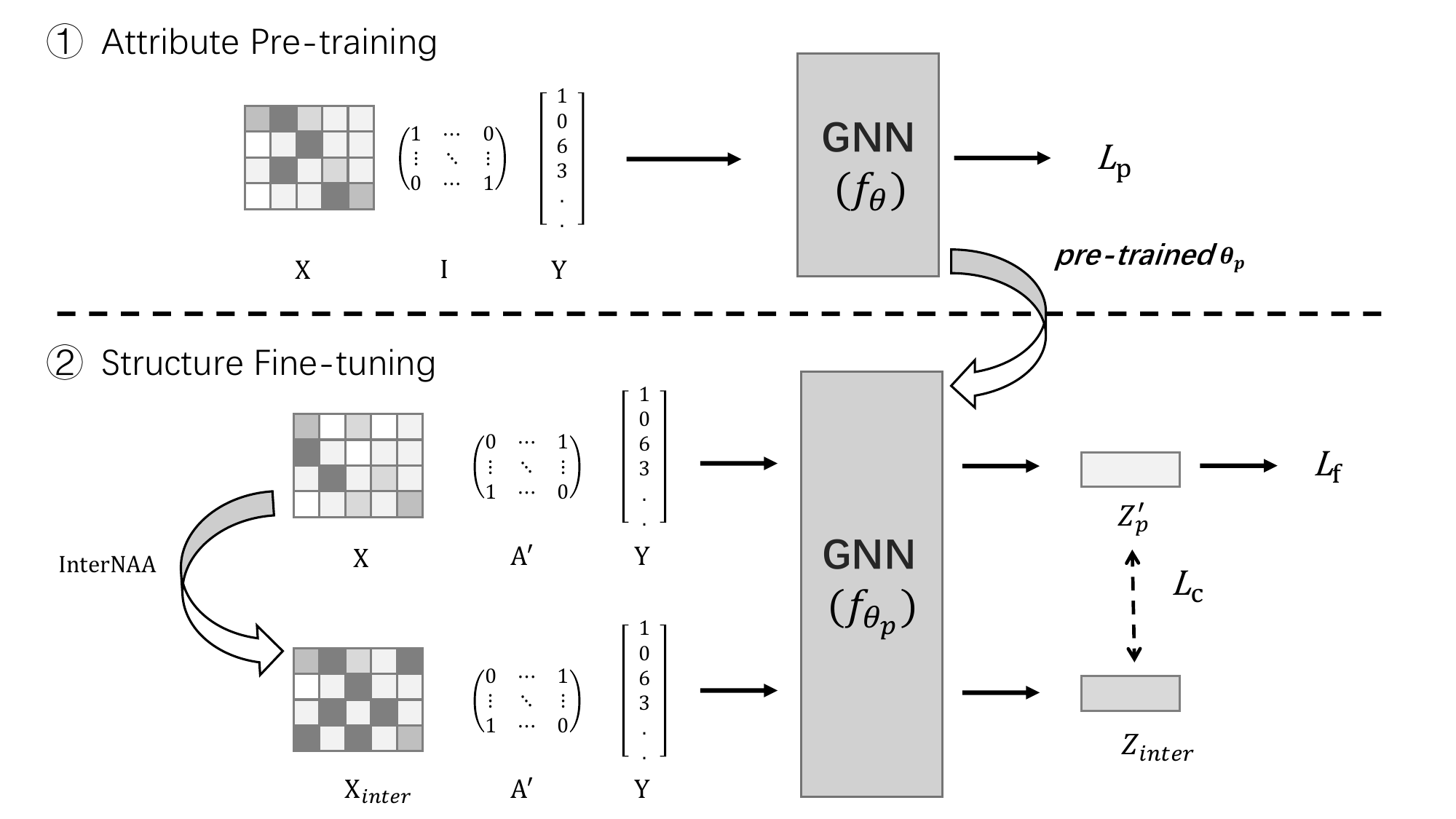}
    \caption{Framework of \ourmethod.}
    \vspace{-0.35cm}
    \label{fig:framework}
\end{figure}

\subsection{Detailed Construction}
To implement the intuition in Sec.~\ref{sec:intuition}, we propose a novel method, namely {\underline Si}mple and {\underline F}ast {\underline R}obust {\underline G}raph {\underline N}eural {\underline N}etwork (\ourmethod) consisting of two main stages: attributes pre-training and structure fine-tuning as shown in Fig.~\ref{fig:framework} and Alg.~\ref{alg:framework}. First, \ourmethod\ pre-trains over the node attributes without structural information to generate node embeddings. Subsequently, \ourmethod\ fine-tunes the node embeddings with the \textit{modified} adjacency matrix to incorporate structural information. 

\paragraph{Attributes Pre-training.}
Attributes pre-training is employed to learn node embeddings solely from node attributes $\mathbf{X}$. Specifically, a randomly initialized GNN model $f_{\theta}$ with parameters $\theta$ takes node attributes $\mathbf{X}$ of the input graph $\mathcal{G}'=(\mathbf{X},\mathbf{A}',\mathbf{Y})$ and an identity matrix $\mathbf{I}$ as inputs and aims to minimize the pre-training loss $\mathcal{L}_{p}$ to learn node embeddings $\mathbf{Z}_{p}$:
\begin{align}
    \theta_{p} =  \underset{\theta}{\operatorname{arg\,min}}\ \mathcal{L}_{p}(f_{\theta}(\mathbf{X},\mathbf{I}), \mathbf{Y}), \quad
     \mathbf{Z}_{p} = f_{\theta_{p}}(\mathbf{X},\mathbf{I}).
\end{align}

The choice of pre-training loss $\mathcal{L}_{p}$ can be any common classification loss, such as the Negative Log-Likelihood Loss function (NLL). Since the pre-training process completely excludes $\mathbf{A'}$ and $\mathbf{Z}_{p}$ is learned from $\mathbf{X}$ without being modified by structural attacks, $\mathbf{Z}_{p}$ is uncontaminated. Lines 3-7 of Alg.~\ref{alg:framework} shows the attributes pre-training stage.

Although the embeddings $\mathbf{Z}_{p}$ learned through attribute pre-training are sufficiently ``clean", the lack of structural information makes $\mathbf{Z}_{p}$ insufficient to predict labels accurately. Hence, we propose structure fine-tuning, which adjusts $\mathbf{Z}_{p}$ using $\mathbf{A}'$ to incorporate some structural information.

\paragraph{Structure Fine-tuning.}

In structure fine-tuning, the model 
is initialized by the pre-trained parameters $\theta_{p}$ and minimizes the fine-tuning loss function $\mathcal{L}_{f}$ and contrastive loss function $\mathcal{L}_{c}$ simultaneously:
\begin{align}
    &\theta^{*}  = \underset{\theta_p}{\operatorname{arg\,min}}\ \mathcal{L}_{f}(\mathbf{Z}_{p}^{'}, \mathbf{Y}) + \mathcal{L}_{c}(\mathbf{Z}_{p}^{'}, \mathbf{Z}_{inter}), \nonumber\\
    &\mathbf{Z}_{p}^{'} =  f_{\theta_{p}}(\mathbf{X_{train}},\mathbf{A}'), \quad \mathbf{Z}_{inter} = f_{\theta_{p}}(\mathbf{X_{inter}},\mathbf{A}'),
\end{align}
where the the fine-tuning loss function $\mathcal{L}_{f}$ is as same as $\mathcal{L}_{p}$, $\mathcal{L}_{c}$ is any typical contrastive function such as InfoNCE. $\mathbf{X_{inter}}$ is generated by the proposed Inter-class Node Attributes Augmentation (InterNAA), which replaces the node feature of each node $v$ in the training set with the average node feature of several nodes with the different class as $v$ that are sampled randomly from the training set. The number of samples equals the degree of the node $v$. The process of InterNAA is shown in Lines 9-17 in Alg.~\ref{alg:framework}.

The primary objectives of the structure fine-tuning stage are twofold: to ensure that $\mathbf{Z}_{p}^{'}$ contains structural information and to prevent it from being influenced by contaminated information in structure ($I(\mathbf{A}';\mathbf{Y}|\mathbf{X})$). The former is achieved by minimizing fine-tuning loss $\mathcal{L}_{f}$, while the latter is achieved by minimizing contrastive loss $\mathcal{L}_{c}$. A detailed theoretical analysis is provided in Sec.~\ref{sec:proof}. Here, we offer an intuitive explanation.

Firstly, the pre-trained parameters $\theta_p$ are used to initialize the model $f$. However, unlike the training stage, $f$ receives $\mathbf{A}'$ instead of $\mathbf{I}$ as input, which means $f$ fine-tunes the pre-trained embeddings $\mathbf{Z}_{p}$ using structure information to obtain $\mathbf{Z}_{p}^{'}$. Besides, by combining contrastive learning techniques to maximize the similarity between $\mathbf{Z}_{p}^{'}$ and $\mathbf{Z}_{inter}$, we effectively align $\mathbf{Z}_{p}^{'}$ with the less harmful $I(\mathbf{A}';\mathbf{Y}|\mathbf{X_{inter}})$ rather than the contaminated $I(\mathbf{A}';\mathbf{Y}|\mathbf{X})$.  $I(\mathbf{A}';\mathbf{Y}|\mathbf{X_{inter}})$ is less harmful 
because we replace $\mathbf{X}$ with $\mathbf{X_{inter}}$ generated by InterNAA, akin to reducing the lethality of a gun by providing it with mismatched bullets.

\begin{algorithm}[tb]
\small
\caption{Simple and Fast Robust Graph Neural Network}
\begin{algorithmic}[1]
    \STATE{\bfseries Input}: Input graph $\mathcal{G}'(\mathbf{X}, \mathbf{A}')$, includes modified adjacency matrix $\mathbf{A}'$ and node attributes $\mathbf{X}$, identity matrix $\mathbf{I}$ and GNN model $f$ with parameter $\theta$, pre-training epoch $pretrain\_epoch$, finetune epoch $finetune\_epoch$, pre-taining loss $\mathcal{L}_{p}$ and fine-tuning loss $\mathcal{L}_{f}$ and contrastive loss $\mathcal{L}_{c}$, node set $V_{train}$, $\mathbf{X}$, $\mathbf{Y}$ for training
    
    \STATE \textit{/*Attributes pre-training*/}
    \FOR{$e=1,2,...,pretrain\_epoch$}
        \STATE $\mathbf{Z}_{\theta} = f_{\theta}(\mathbf{X}, \mathbf{I})$
        \STATE $\theta =\theta + \nabla_{\theta}$ $\mathcal{L}_{p}(\mathbf{Z}_{\theta}, \mathbf{Y})$
    \ENDFOR
    \STATE $\theta_{p}$ = $\theta$
    \STATE \textit{/*Inter-class Node Attributes Augmentation*/}
    \STATE empty set $\mathbf{X_{inter}}$.
    \FOR{$v\in V_{train}$}
        \STATE $c = v.class$
        \STATE InterClassSet=$\{V_{train}.class\neq c\}$
        \STATE num = v.degree
        \STATE $inter\_class=$RandomChoice(num,InterClassSet)
        \STATE $\mathbf{X}_{v\_inter} = mean(\mathbf{X}_{inter\_class})$
        \STATE add $\mathbf{X}_{v\_inter}$ into $\mathbf{X_{inter}}$
    \ENDFOR
    \STATE \textit{/*Structure fine-tuning*/}
    \FOR{$e=1,2,...,finetune\_epoch$}

        \STATE $\mathbf{Z}_{p}^{'} = f_{\theta_{p}}(\mathbf{X}, \mathbf{A}')$, \quad $\mathbf{Z}_{inter}=f_{\theta_{p}}(\mathbf{X_{inter}}, \mathbf{A}')$
        \STATE $\mathcal{L} = \mathcal{L}_{f}(\mathbf{Z}_{\theta_{p}}^{'},\mathbf{Y}) + \mathcal{L}_{c}(\mathbf{Z}_{p}^{'},\mathbf{Z}_{inter})$
        \STATE $\theta_{p} = \theta_{p} + \nabla_{\theta_{p}}\mathcal{L}$
    \ENDFOR
    \STATE $\theta^{*} = \theta_{p}$
    \RETURN $f_{\theta^{*}}$
\end{algorithmic}
\label{alg:framework}
\end{algorithm}

\paragraph{Computational Complexity.}
The computational complexity consists of two parts: the attribute pre-training stage (Lines 3-7 in Alg.~\ref{alg:framework}) and the structure fine-tuning stage (Lines 19-24 in Alg.~\ref{alg:framework}). Assuming our network is composed of $L$ layers of graph convolutional layers, with $F$ hidden units per layer, $N$ nodes and $E$ edges in the graph, pretraining epochs $e_p$, and finetuning epochs $e_f$. Since the attributes pretraining does not utilize the adjacency matrix $\mathbf{A}'$, its computational complexity can be considered equivalent to that of a multi-layer perceptron (MLP), which is $\mathrm{O}(e_{p}LNF^{2})$. 

As for structure fine-tuning, the nodes in the training set are traversed to generate $\mathbf{X_{inter}}$ (Lines 9-17 in Alg.~\ref{alg:framework}), with computational complexity of $\mathrm{O}(\sigma N\bar{d}F)$, where $\sigma$ is the training ratio and $\bar{d}$ is the average degrees. The computational complexity of obtaining $\mathbf{Z}_{p}^{'}$ and $\mathbf{Z}_{inter}$, is equal to applying twice calculations of GCNs: $\mathrm{O}(2*(LNF^{2}+LEF))$ \cite{chen2020scalable}. The computational complexity for computing the contrastive loss is $\mathrm{O}(\sigma ^{2}N ^{2}F)$ \cite{zhang2022esco, alon2024optimal}. 

Thus the overall computational complexity of \ourmethod\ is $\mathrm{O}(e_{p}LNF^{2}+\sigma N\bar{d}F+e_{f}(2LNF^{2}+2LEF)+\sigma ^{2}N ^{2}F)$.
In the worst case when the graph is fully connected, where $\bar{d}=N$ and $E=N^2$, the complexity is 
$\mathrm{O}(e_{p}LNF^{2}+\sigma N^{2}F+e_{f}(2LNF^{2}+2LN^{2}F+\sigma ^{2}N ^{2}F))$. Since the training ratio $\sigma$ is less than $1$, $e_{p}$ and $e_{f}$ are constants smaller than $N$, and their impact on the overall complexity is negligible. Hence, the total complexity of \ourmethod\ is $\mathrm{O}(LNF^{2}+LN^{2}F)$, which is on par with that of GCN, and significantly lower than that of existing robust GNNs. The experiments in Sec.~\ref{sec:experiments} substantiate this claim.

\section{Theoretical Analysis}
\label{sec:proof}





Our theoretical analysis serves two purposes: first, to analyze the essence of structural attacks and the paired effect from the perspective of mutual information, providing a theoretical explanation for our intuition; second, to theoretically prove the effectiveness of our proposed "attributes pre-training, structure fine-tuning" strategy. 

Given the fundamental properties of mutual information, performance degradation of GNN can be attributed to the maliciously generated adjacency matrix over true node attributes.
Accordingly, we  provide the understanding of structural attacks from an information-theoretic perspective as in Lemma.~\ref{lemma:1}.

\begin{lemma}[\textbf{Essence of Structural Attacks}]\label{lemma:1}
 Structural attacks degrade GNNs’ performance through generating the modified adjacency matrix $\mathbf{A}'$ to contaminate the mutual information between the labels $\mathbf{Y}$ and $\mathbf{A}'$ conditioned by $\mathbf{X}$,  which essentially uses the mutual information $I(\mathbf{A}'; \mathbf{Y}|\mathbf{X})$. 
\end{lemma}
The significance of Lemma.~\ref{lemma:1} lies in highlighting that structural attacks essentially generate modified structure $\mathbf{A}'$ according to corresponding node attributes $\mathbf{X}$, which implies the potential relationships between $\mathbf{A}'$ and $\mathbf{X}$.
Building on Lemma.~\ref{lemma:1}, we propose Lemma.~\ref{lemma:2} blow to demonstrate the important correspondence between $\mathbf{A}'$ and $\mathbf{X}$ in $I(\mathbf{A}'; \mathbf{Y}|\mathbf{X})$.
Namely, $I(\mathbf{A}'; \mathbf{Y}|\mathbf{X})$ can only maximally degrade GNN performance under the condition of $\mathbf{X}$.
\begin{lemma}[\textbf{Paired Effect of Structural Attacks}]\label{lemma:2}
    For any $\mathbf{X}'\neq \mathbf{X}$, where $\mathbf{X}',\mathbf{X}\in\mathbb{R}^{n\times d}$, the  mutual information $I(\mathbf{A}'; \mathbf{Y}|\mathbf{X}')$ is less harmful to GNNs than $I(\mathbf{A}'; \mathbf{Y}|\mathbf{X})$.
\end{lemma}
Notably, Lemma.~\ref{lemma:2} implies a new defense strategy against structural attacks from the root cause. Unlike existing methods that focus on purifying the modified structure or employing adaptive aggregation, our approach does not require any operations on the modified structure. Instead, it replaces the corresponding attributes to disrupt the paired effect, thereby reducing the attack effectiveness of the modified structure on GNNs.

Motivated by Lemma.~\ref{lemma:2}, \ourmethod\  pre-trains node embeddings $\mathbf{Z}_{p}$ solely on node attributes $\mathbf{X}$, and force the proposed model to learn information from $I(\mathbf{A}';\mathbf{Y}|\mathbf{Z}_{p})$, which actually replaces $\mathbf{X}$ with $\mathbf{Z}_{p}$. We provide Theorem.~\ref{theorem:1} to demonstrate $\mathbf{Z}_{p}$ shares mutual information with labels $\mathbf{Y}$ and $I(\mathbf{A}';\mathbf{Y}|\mathbf{Z}_{p})$ is less harmful compared to $I(\mathbf{A}';\mathbf{Y}|\mathbf{X})$.

\begin{theorem}\label{theorem:1}
\ourmethod's pre-training stage maximizes the mutual information $I(\mathbf{Z}_{p};\mathbf{Y})$ between $\mathbf{Z}_{p}$ and $\mathbf{Y}$, where $I(\mathbf{A}';\mathbf{Y}|\mathbf{Z}_{p})$ is less harmful to GNNs compared to $I(\mathbf{A}';\mathbf{Y}|\mathbf{X})$.
\end{theorem}

Although $I(\mathbf{A}';\mathbf{Y}|\mathbf{Z}_{p})$ is less harmful than $I(\mathbf{A}';\mathbf{Y}|\mathbf{X})$, there may be an overlap between them since $\mathbf{Z}_{p}$ is learned from $\mathbf{X}$, leading to the contamination of $I(\mathbf{A}';\mathbf{Y}|\mathbf{Z}_{p})$, as demonstrated in Lemma.\ref{lemma:3}.
Lemma.\ref{lemma:3} essentially explains the reason for employing contrastive learning, i.e., $I(\mathbf{A}';\mathbf{Y}|\mathbf{Z}_{p})$ may be contaminated. Based on Lemma.\ref{lemma:3}, we introduce contrastive learning to align $I(\mathbf{A}';\mathbf{Y}|\mathbf{Z}_{p})$ to $I(\mathbf{A}';\mathbf{Y}|\mathbf{X_{inter}})$ to prevent it from being contaminated, as demonstrated in Theorem.~\ref{theorem:2}.

\begin{lemma}\label{lemma:3}
    There exists an overlap between $I(\mathbf{A}';\mathbf{Y}|\mathbf{Z}_{p})$ and $I(\mathbf{A}';\mathbf{Y}|\mathbf{X})$, which consequently leads to the contamination of $I(\mathbf{A}';\mathbf{Y}|\mathbf{Z}_{p})$.
\end{lemma}

\begin{theorem}\label{theorem:2}
    \ourmethod's structure fine-tuning stage maximizes $I(\mathbf{A}';\mathbf{Y}| \mathbf{Z}_{p}^{'})$ to learn structural information and align it to $I(\mathbf{A}';\mathbf{Y}| \mathbf{X}_{inter})$ to prevent from being contaminated.
\end{theorem}

Theorem.\ref{theorem:2} demonstrates the effectiveness of contrastive learning and the proposed InterNAA to mitigate the contamination. We provide complete proofs in the Appendix and present empirical experiments in Sec.~\ref{sec:experiments} to support the aforementioned claims.

\begin{table}[t]
\scriptsize
\centering
\begin{tabular}{|c|c|c|}
\hline
Dataset                      & hyper-parameters & Range                       \\ \hline
\multirow{2}{*}{SimP-GCN}    & $\lambda$        & \{0.1,1,10\}                \\ \cline{2-3} 
                             & $\gamma$         & \{0.01,0.1,1\}              \\ \hline
\multirow{4}{*}{ProGNN}      & $\alpha$         & 0.00025 to 0.064            \\ \cline{2-3} 
                             & $\beta$          & 0 to 5                      \\ \cline{2-3} 
                             & $\gamma$         & 0.0625 to 16                \\ \cline{2-3} 
                             & $\lambda$        & 1.25 to 320                 \\ \hline
\multirow{5}{*}{STABLE}      & $k$              & 1 to 13                     \\ \cline{2-3} 
                             & $\alpha$         & -0.5 to 3                   \\ \cline{2-3} 
                             & $t_1$            & 0 to 0.05                   \\ \cline{2-3} 
                             & $t_2$            & -1 to 0.5                   \\ \cline{2-3} 
                             & $p$              & 0.2                         \\ \hline
\multirow{2}{*}{EvenNet}     & $k$              & \{4,6,8,10\}                \\ \cline{2-3} 
                             & $\alpha$         & 0 to 0.5                    \\ \hline
\multirow{2}{*}{NoisyGCN}    & $\beta$          & 0.1 to 0.5                  \\ \cline{2-3} 
                             & $\epsilon$          & \{0, 0.1\}                  \\ \hline
\multirow{3}{*}{GADC}        & $\xi$            & 0.01 to 1                   \\ \cline{2-3} 
                             & $K$              & \{4,8,16,32\}               \\ \cline{2-3} 
                             & $\lambda$        & \{4,8,16,32\}               \\ \hline
\end{tabular}
\caption{Hyper-parameters of baselines and their choices. }
\label{tab:hyper}
\end{table}

\section{Experiments}
\label{sec:experiments}


\begin{table*}[ht]
\centering
\scriptsize
\begin{tabular}{|c|c|c|cccccccc|c|}
\hline
Datasets & Attacker & ptb(\%)                                                            & RGCN      & GCN-Jaccard & SimP-GCN  & Pro-GNN   & STABLE    & EvenNet   & GADC      & NoisyGCN  & \ourmethod\ (Ours) \\ \hline
\multicolumn{1}{|c|}{\multirow{7}{*}{Cora}}     & \multicolumn{2}{c|}{clean}                         & \textbf{83.4±0.2} & 82.2±0.5       & 82.1±0.6          & 83.0±0.2          & 81.1±0.5          & \underline{83.1±0.4} & 79.0±0.3          & 82.9±0.6          & \textbf {83.4±0.5}    \\\cline{2-3}
\multicolumn{1}{|c|}{}                          & \multicolumn{1}{c|}{\multirow{3}{*}{Mettack}} & 0.05 & 72.0±0.5          & 78.8±0.6       & 80.5±1.7          & \underline{82.3±0.5}    & 81.4±0.6          & 80.0±0.8 & 78.7±0.3          & 77.4±1.2          & \textbf{82.6±0.6} \\
\multicolumn{1}{|c|}{}                          & \multicolumn{1}{c|}{}                      & 0.1  & 68.9±0.3          & 76.9±0.5       & 79.0±0.9          & 79.0±0.6          & \underline{81.0±0.4}    & 77.8±1.1 & 78.2±0.6          & 75.6±1.2          & \textbf{82.1±0.6} \\
\multicolumn{1}{|c|}{}                          & \multicolumn{1}{c|}{}                      & 0.2  & 62.8±1.2          & 75.2±0.7       & 76.1±2.0          & 73.3±1.6          & \underline{80.4±0.7}    & 78.2±0.9 & 77.1±0.6          & 74.5±1.3          & \textbf{81.1±0.8} \\\cline{2-3} 
\multicolumn{1}{|c|}{}                          & \multicolumn{1}{c|}{\multirow{3}{*}{GraD}} & 0.05 & 81.7±0.5          & 81.4±0.6       & 80.9±0.5          &\underline{81.8±0.5}    & 81.1±0.4          & 80.4±0.5 & 79.0±0.3          & \textbf{82.0±0.6} & \textbf{82.0±0.6} \\
\multicolumn{1}{|c|}{}                          & \multicolumn{1}{c|}{}                      & 0.1  & 79.9±0.4          & 80.3±0.4       & 80.9±0.5          & 80.9±0.3          & 80.2±0.5 & 78.7±0.7 & 78.7±0.3          & \underline{81.0±0.3}          & \textbf{81.1±0.8}    \\
\multicolumn{1}{|c|}{}                          & \multicolumn{1}{c|}{}                      & 0.2  & 77.9±0.5          & \textbf{79.8±0.6} & 78.9±0.8          & 78.3±0.2          & \textbf{79.8±0.3} & 78.3±1.0 & 78.6±0.4          & 79.1±0.6          & \underline{79.6±0.9}          \\ \hline
\multicolumn{1}{|c|}{\multirow{7}{*}{Citeseer}} & \multicolumn{2}{c|}{clean}                        & 71.8±0.6          & 72.6±0.6       & \underline{73.8±0.7}    & 73.3±0.7          & \underline{73.9±0.6}          & 73.8±0.5 & 73.4±0.5          & 72.3±0.4          & \textbf{75.1±0.4} \\\cline{2-3}  
\multicolumn{1}{|c|}{}                          & \multicolumn{1}{c|}{\multirow{3}{*}{Mettack}} & 0.05 & 70.5±1.0          & 72.0±0.4       & 73.0±0.7    & 72.9±0.6          & 72.6±0.3          & \underline{73.5±0.4} & 73.0±0.8          & 71.3±0.3          & \textbf{74.7±0.5} \\
\multicolumn{1}{|c|}{}                          & \multicolumn{1}{c|}{}                      & 0.1  & 69.4±0.8          & 71.8±0.5       & \underline{74.1±0.7}    & 72.5±0.8          & 73.5±0.5          & 73.3±0.4 & 73.0±0.8          & 71.2±0.4          & \textbf{74.3±0.4} \\
\multicolumn{1}{|c|}{}                          & \multicolumn{1}{c|}{}                      & 0.2  & 67.7±0.5          & 70.6±0.3       & 70.9±0.5          & 70.0±2.3          & 72.8±0.7          & \underline{73.2±0.5} & 72.9±0.8    & 70.2±0.7          & \textbf{73.7±0.4} \\\cline{2-3} 
\multicolumn{1}{|c|}{}                          & \multicolumn{1}{c|}{\multirow{3}{*}{GraD}} & 0.05 & 71.6±0.8          & 72.1±0.9       & 73.5±0.7    & 72.2±0.1          & 73.5±0.4          & \underline {73.8±0.8} & 73.5±0.8          & 72.0±0.8          & \textbf{75.0±0.6} \\ 
\multicolumn{1}{|c|}{}                          & \multicolumn{1}{c|}{}                      & 0.1  & 70.7±0.6          & 72.3±0.7       & \underline {73.5±0.6}    & 72.1±0.1          & 72.5±0.6          & 73.0±0.5 & 73.4±0.8          & 71.8±0.4          & \textbf{74.5±0.6} \\ 
\multicolumn{1}{|c|}{}                          & \multicolumn{1}{c|}{}                      & 0.2  & 67.6±0.6          & 70.1±0.8       & 72.6±0.7          & 70.6±0.6          & 72.1±0.5          & 72.7±0.5 & \underline {73.2±0.7}    & 70.4±0.5          & \textbf{73.4±0.6} \\ \hline
\multicolumn{1}{|c|}{\multirow{7}{*}{Pubmed}}   & \multicolumn{2}{c|}{clean}                        & 85.4±0.1          & 86.2±0.1       & \underline {87.1±0.1}    & \textbf{87.3±0.2} & 85.0±0.1          & 86.7±0.1 & 86.4±0.1          & 85.0±0.0          & 85.4±0.4          \\\cline{2-3} 
\multicolumn{1}{|c|}{}                          & \multicolumn{1}{c|}{\multirow{3}{*}{Mettack}} & 0.05 & 83.0±0.2          & 83.6±0.5       & \underline {86.5±0.1}    & \textbf{87.2±0.1} & 81.3±0.1          & 86.0±0.2 & 86.3±0.1          & 79.7±0.2          & 85.3±0.4          \\
\multicolumn{1}{|c|}{}                          & \multicolumn{1}{c|}{}                      & 0.1  & 83.0±0.2          & 79.6±0.2       & 86.0±0.2          & \textbf{87.2±0.1} & 79.0±0.1          & 85.6±0.2 & \underline {86.3±0.2}    & 67.4±0.2          & 85.1±0.3          \\ 
\multicolumn{1}{|c|}{}                          & \multicolumn{1}{c|}{}                      & 0.2  & 81.4±0.2          & 70.5±0.4       & \underline{ 85.7±0.2}    & \textbf{87.2±0.2} & 78.4±0.1          & 85.3±0.2 & 86.1±0.1          & 56.5±0.4          & 84.5±0.4          \\\cline{2-3}  
\multicolumn{1}{|c|}{}                          & \multicolumn{1}{c|}{\multirow{3}{*}{GraD}} & 0.05 & 82.9±0.1          & 84.0±0.2       & \textbf{86.6±0.2} & 85.4±0.0          & 82.6±0.1          & 86.2±0.2 & \underline{ 86.3±0.1}    & 82.8±0.1          & 85.2±0.3          \\
\multicolumn{1}{|c|}{}                          & \multicolumn{1}{c|}{}                      & 0.1  & 81.8±0.1          & 82.7±0.1       & \underline{ 86.0±0.2}    & 85.0±0.2          & 81.5±0.1          & 85.9±0.2 & \textbf{86.1±0.1} & 81.7±0.1          & 84.5±0.5          \\ 
\multicolumn{1}{|c|}{}                          & \multicolumn{1}{c|}{}                      & 0.2  & 79.6±0.1          & 80.5±0.2       & \underline{ 85.3±0.4}    & 82.8±0.1          & 79.2±0.1          & 85.4±0.2 & \textbf{86.1±0.1} & 80.0±0.1          & 83.4±0.4          \\ \hline
\end{tabular}
\caption{Average classification accuracy (± standard deviation) of 10 runs under two structural attacks with different perturbation ratios (ptb). The best and second-best results are highlighted in bold and underlined, respectively.}
\label{tab1}
\end{table*}


\paragraph{Datasets. } We conduct experiments on three widely used benchmarks: Cora \cite{cora}, CiteSeer \cite{citeseer}, Pubmed \cite{pubmed}, and two large-scale graph datasets (ogbn-arxiv, and ogbn-products), with details presented in the Appendix. Furthermore, experimental results on two heterophilic graph datasets, demonstrating the robustness of the proposed method, are provided in the Appendix.

\paragraph{Implementation and Baselines. }We conducted an empirical comparison against eight state-of-the-art baseline defense algorithms, including RGCN \cite{RGCN}, GCN-Jaccard \cite{GCN-Jaccard}, SimP-GCN \cite{SimP-GCN}, Pro-GNN \cite{Pro-GNN}, STABLE \cite{STABLE}, EvenNet \cite{EvenNet}, GADC \cite{GADC}, and NoisyGCN \cite{NoisyGCN}, which achieve remarkable performance in terms of structure attack defense. 
We select two representative structure attack methods, i.e., Mettack \cite{Mettack} and GraD \cite{GraD}, to verify the robustness of the proposed method and baselines. Source code and configuration of baselines are obtained from either the public implementation of DeepRobust \cite{deeprobust}, or the official implementation of the authors. Detailed configurations are deferred to the Appendix.

\paragraph{Experiments Settings. } Experiments are conducted on a device with 16 Gen Intel(R) Core(TM) i9-12900F cores and an NVIDIA L20 (48GB memory). On Cora, Citeseer and Pubmed, we follow the data splitting method of DeepRobust: randomly selecting 10\% of the nodes for training, 10\% for validation, and the remaining 80\% for testing. As for ogbn-arxiv and ogbn-products, we follow dataset splits provided by OGB \cite{ogbn}. As for hyper-parameters of baselines, we follow the authors’ suggestion to search for the optimal values. Table.~\ref{tab:hyper} shows all hyper-parameters and their ranges. It can be observed that existing robust GNNs require multiple hyper-parameters, and some of them have a large search range. Consequently, existing robust GNNs require many training runs to determine the optimal values of all hyper-parameters.

\begin{table*}[t]
    \centering
    \small
    \begin{tabular}{|c|c|c|c|c|c|c|c|}
    \hline
        Dataset & ptb(\%) & GCN & GADC & EvenNet & Soft Medoid GDC & Soft Median GDC & \makecell[c]{\ourmethod\ \\ (w/o CL)} \\ \hline
        \multirow{5}{*}{ogbn-arxiv}
         & clean & \underline{66.9+0.32} & 64.1±0.15 & 63.2±1.3 & 57.5±0.24 & 64.1±0.15 & \textbf{67.0±0.22} \\
        ~ & 1\% & 54.8+0.29 & 55.6±0.13 & 36.4±7.92 & 52.2±0.22 & \underline{56.9±0.19} & \textbf{58.8±0.16} \\
        ~ & 5\% & 34.6±0.32 & 45.2±0.17 & 32.4±4.94 & \underline{48.0±0.27} & 47.1±0.21 & \textbf{48.5±0.21} \\
        ~ & 10\% & 29.5±0.58 & 36.4±0.19 & 29.2±2.45 & \textbf{45.4±0.31} & 40.8±0.33 & \underline{41.5±0.21} \\ \cline{2-8}
        ~ & speed(ms/epoch) & 109.15 & 119.8 & \underline{89.9} & 145.3 & 132.7 & \textbf{53.0} \\ \hline  
        \multirow{5}{*}{ogbn-products} & clean & \underline{73.5±0.08} & 73.0±0.05 & \multirow{5}{*}{OOM} & \multirow{5}{*}{OOM} & 64.3±0. & \textbf{74.0±0.28} \\
        & 1\% & 63.6±0.10 & \underline{66.6±0.12} & ~ & ~ & 63.0±0.08 & \textbf{69.9±0.30} \\
        & 5\% & 49.5±0.09 & 58.7±0.17 & ~ & ~ & \underline{59.0±0.11} & \textbf{59.6±0.3} \\
        & 10\% & 46.3±0.11 & 52.5±0.26 & ~ & ~ & \textbf{56.9±0.14} & \underline{54.2±0.26} \\ \cline{2-4} \cline{7-8}
        & speed(ms/epoch) & 367.7 & \underline{355.2} & ~ & ~ & 413.0 & \textbf{190.2} \\ \hline
    \end{tabular}
    \caption{Average classification accuracy (± standard deviation) and average training speed (in milliseconds per epoch) of 10 runs on large-scale graph datasets under PRBCD attacks.}
    \label{tab:largescale}
\end{table*}

\paragraph{Defense Performance.}

To showcase the effectiveness and efficiency of the proposed method, we compare its robustness (Table.~\ref{tab1}) and training time (Table.~\ref{tab:training time}) against two typical attack methods: Mettack and GraD, on three datasets with baselines. 
It's worth noting that the proposed method always achieves the best performance or the second-best performance on Cora and Citeseer, highlighting its robustness, which is on par with or exceeds state-of-the-art baselines. For instance, the proposed method achieves $82.1\%$ accuracy on Cora dataset under Mettack with $10\%$ perturbation ratio while baselines' accuracy ranges from $69\%$ to $81\%$. Besides, for Citeseer, the proposed method achieves tiny but continuous improvements compared to baselines under all perturbation ratios. On the Pubmed dataset, while \ourmethod\ does not surpass strong baselines like SimP-GCN and Pro-GNN, it still outperforms several other baselines. We speculate that the reason is more complex and larger models tend to have an advantage on larger datasets like Pubmed.

Besides the robustness improvements, the proposed method also achieves significant training time speedup compared to baselines as shown in Tabel.~\ref{tab:training time}. Upon examining the table, we can observe that compared to the fastest existing methods in their respective categories, such as NoisyGCN and GADC, the proposed method achieves over a 100\% speedup on Cora, and Citeseer. Conversely, when compared to slower methods like SimP-GCN and STABLE, the proposed method's speed is nearly 10 times that of theirs. The significant speedup achieved by the proposed method can be attributed to the elimination of time-consuming modified structure identification and processing operations.

\paragraph{Scalability to Large-scale Graph.}

\begin{table}[!h]
\centering
\small
\begin{tabular}{|c|c|c|c|}
\hline
Dataset & Cora & Citeseer & Pubmed \\ \hline
RGCN & 3.56 & 3.57 & 41.39  \\
GCN-Jaccard & 3.34 & 3.28 & 13.66  \\
SimP-GCN & 10.42 & 9.89 & 37.68  \\
ProGNN & 4.21 & 4.89 & \textgreater 1,000  \\ 
STABLE & 7.26 & 5.52 & 64.10  \\
EvenNet & 4.70 & 4.65 & 4.75  \\
NoisyGCN & 3.19 & 3.55 & \underline{4.47}  \\
GADC & \underline{2.32} & \underline{3.25} & 6.65  \\
\hline
\ourmethod\  & \makecell[c]{\textbf{1.11}\\($\uparrow 109\%$)} & \makecell[c]{\textbf{1.24}\\($\uparrow 162\%$)} & \makecell[c]{\textbf{3.61}\\($\uparrow 24\%$)} \\ \hline
\end{tabular}
\caption{Average Training Time Comparison (ms/epoch).}
\label{tab:training time}
\end{table}

We conduct experiments on two publicly available large-scale graph datasets, ogbn-arxiv and ogbn-products \cite{ogbn}, to validate the scalability of the proposed method. Owing to memory overflow issues encountered by structure attack methods like Mettack on large-scale graphs, we employ PRBCD \cite{prbcd} as the attack method for these settings, and compare its performance against four defense methods capable of scaling to large graphs. The tests are performed using the officially provided modified adjacency matrices, with the results presented in Table.~\ref{tab:largescale}. Given the substantial memory requirements of the contrastive learning component, to facilitate the scalability of \ourmethod\  to large-scale graphs, we introduce a variant of our approach: \ourmethod\ (w/o CL), which excludes contrastive learning during the structure fine-tuning stage, thereby enabling its effective application to large-scale graphs without encountering memory constraints.

Results in Table.~\ref{tab:largescale} demonstrate that \ourmethod\  consistently achieves either the top or second-best performance across various perturbation ratios on two large-scale datasets. Notably, it also exhibits the fastest runtime, surpassing even GCN in speed. This efficiency stems from the attributes pre-training stage of \ourmethod\ , which is free from structure information related computations. Additionally, EvenNet and Soft Medoid GDC encountered out-of-memory (OOM) issues on ogbn-products. This is attributed to the fact that Soft Medoid GDC incorporates diffusion computations, rendering it less scalable \cite{prbcd}, while EvenNet demands a minimum of 70GB of GPU memory, exceeding the capacity of our experiment device, which is limited to 48GB. It's worth noting that the speed advantage of \ourmethod\  is particularly pronounced in large-scale graphs compared to tiny graphs, which demonstrates the simplicity and effectiveness of \ourmethod.

\paragraph{Ablation Study.}

To validate the effectiveness of the proposed method, we propose two variants: 1) \textbf{\ourmethod\  w/o CL}: This variant lacks the contrastive learning technique and directly fine-tunes the model using the modified adjacency matrix. 2) \textbf{\ourmethod\ w/o Fin}: This variant lacks the whole structure fine-tuning stage, thereby degenerating into a Multilayer Perceptron (MLP). 
The results in Fig.~\ref{fig:ablation} (a) and Fig.~\ref{fig:ablation} (b) show that the accuracies of both variants are consistently lower than that of \ourmethod\ across all attack ratios. \ourmethod\ w/o CL can not performer \ourmethod\ and achieves suboptimal results which proves the Lemma.~\ref{lemma:3} and Theorem.~\ref{theorem:2} in Sec.~\ref{sec:proof}. \ourmethod\ w/o Fin achieved the worst results because the learned representations only contained attribute information without structure information.

Additionally, to validate the effectiveness of the proposed InterNAA, we replace it with commonly used augmentations: Node Dropping (\ourmethod\ w/ND), Edge Removing (\ourmethod\ w/ER), and Feature Masking (\ourmethod\ w/FM). Besides, we provide a variant \ourmethod\ w/Ran which replaces InterNAA with random node attributes sampling. The comparison results are shown in Fig.~\ref{fig:ablation} (a) and Fig.~\ref{fig:ablation} (b). Clearly, InterNAA outperforms other augmentations. We believe this is because other augmentations randomly perturb the elements of the graph and cannot prevent the inflect of contaminated mutual information during the fine-tuning stage. Additionally, to validate the effectiveness of the proposed InterNAA, we replace it with commonly used augmentations: Node Dropping (\ourmethod\ w/ND), Edge Removing (\ourmethod\ w/ER), Feature Masking (\ourmethod\ w/FM) and random node attribute sampling (\ourmethod\ w/Ran). The comparison results are shown in Fig.~\ref{fig:ablation} (a) and Fig.~\ref{fig:ablation} (b). Clearly, InterNAA outperforms the other augmentations. We believe this is because other augmentations randomly perturb the elements of the graph and cannot prevent the influence of contaminated mutual information during the fine-tuning stage.

\begin{figure}[t]
    \centering
    \includegraphics[scale=0.4]{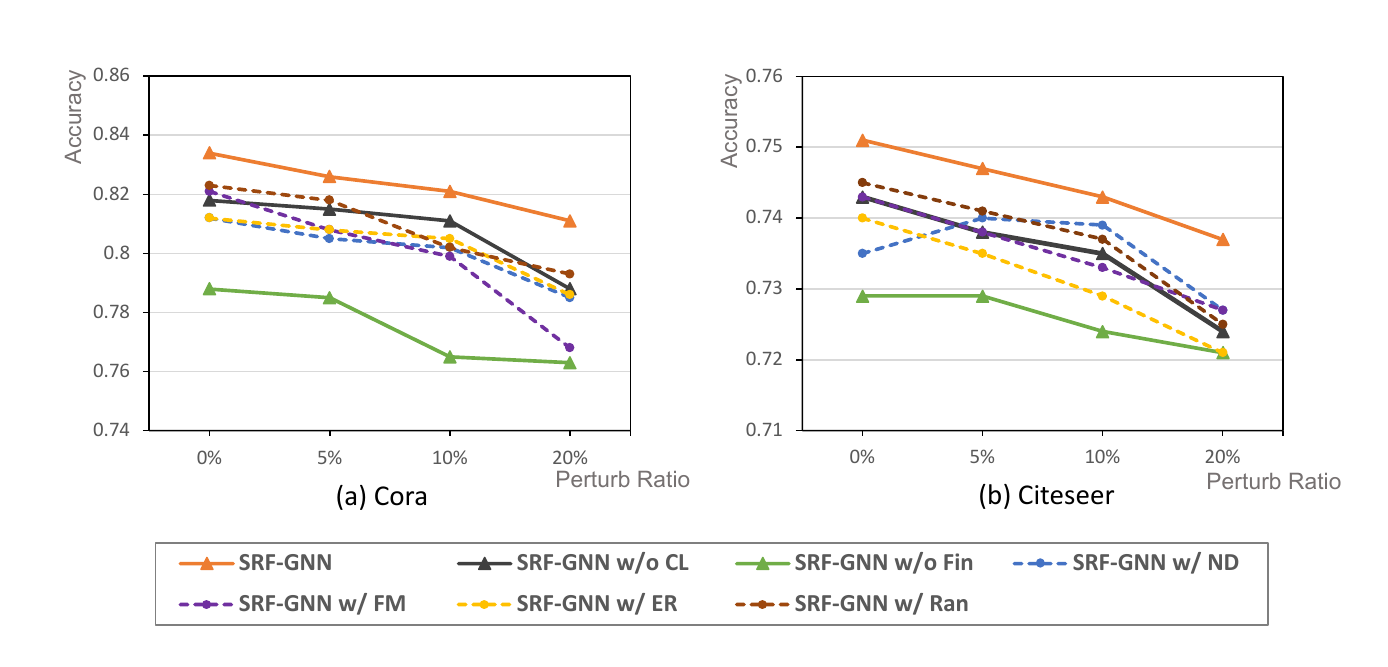}
    \caption{Ablation studies of SFR-GNN. }
    \label{fig:ablation}
\end{figure}

\section{Conclusion}
\label{sec:conclusion}
In this paper, we propose a novel robust GNN: Simple and Fast Robust Graph Neural Network (\ourmethod) against structural attacks. \ourmethod\ utilizes the proposed ``attributes pre-training and structure fine-tuning'' strategy, without the need for purification of the modified structures, thus significantly reducing computational overhead and avoiding the introduction of additional hyper-parameters. We conduct both theoretical analysis and numerical experiments to validate the effectiveness of \ourmethod\ . Experimental results demonstrate that \ourmethod\ achieves robustness comparable to state-of-the-art baselines while delivering a 50\%-136\% improvement in runtime speed. Additionally, it exhibits superior scalability on large-scale datasets. This makes \ourmethod\ a promising solution for applications requiring reliable and efficient GNNs in adversarial settings.

\appendix
\section{A. Theoretical Analyses}
\label{appendix:proofs}

\subsection{A.1 Proof of Lemma.~1}
\label{proof:theorem1}

\begin{lemma}
 Structural attacks degrade GNNs’ performance through generating the modified adjacency matrix $\mathbf{A}'$ to contaminate the mutual information between the labels $\mathbf{Y}$ and $\mathbf{A}'$ conditioned by $\mathbf{X}$,  which essentially uses the mutual information $I(\mathbf{A}'; \mathbf{Y}|\mathbf{X})$. 
\end{lemma}

\begin{proof}
    For a GNN model $f_{\theta}$ parameterized by $\theta$, the objective is to predict labels $\mathbf{Y}$ as accurately as possible by taking node features $\mathbf{X}$ and adjacency matrix $\mathbf{A}$ as inputs. From the perspective of information theory, this objective can be viewed as minimizing the conditional entropy $H(\mathbf{Y}|f_{\theta}(\mathbf{X},\mathbf{A}))$:
    \begin{equation}
        \min\ \mathcal{L}_{tr}(f_{\theta}(\mathbf{X},\mathbf{A}),\mathbf{Y}) \quad \Rightarrow \quad \min\ H(\mathbf{Y}|f_{\theta}(\mathbf{X},\mathbf{A})).
    \end{equation}

    The conditional entropy $H(\mathbf{Y}|f_{\theta}(\mathbf{X},\mathbf{A}))$ measures the uncertainty of $\mathbf{Y}$ given the $f_{\theta}(\mathbf{X},\mathbf{A})$. An effective $f_{\theta}(\mathbf{X},\mathbf{A})$ should be able to predict $\mathbf{Y}$ with high probability, meaning uncertainty is low. Thus the above equation holds.

    According to the principles of mutual information, we have:
    \begin{equation}
        I(f_{\theta}(\mathbf{X},\mathbf{A});\mathbf{Y}) = H(\mathbf{Y}) - H(\mathbf{Y}|f_{\theta}(\mathbf{X},\mathbf{A}).
    \end{equation}
    
    $H(\mathbf{Y})$ is the information entropy of labels which is determined by $\mathbf{Y}$ and fixed. Thus minimizing $H(\mathbf{Y}|f_{\theta}(\mathbf{X},\mathbf{A}))$ is actually maximizing $I(f_{\theta}(\mathbf{X},\mathbf{A});\mathbf{Y})$:
    \begin{equation}
        \min\ H(\mathbf{Y}|f_{\theta}(\mathbf{X},\mathbf{A}))\Rightarrow \quad \max\ I(f_{\theta}(\mathbf{X},\mathbf{A});\mathbf{Y}).
    \end{equation}
    
    Therefore, the learning objective of GNN is actually maximizing the mutual information $I(f_{\theta}(\mathbf{X},\mathbf{A});\mathbf{Y})$:
    \begin{equation}
        \min\ \mathcal{L}_{tr}(f_{\theta}(\mathbf{X},\mathbf{A}),\mathbf{Y})\Rightarrow \quad \max\ I(f_{\theta}(\mathbf{X},\mathbf{A});\mathbf{Y}).
    \end{equation}

    maximizing the mutual information between the labels $\mathbf{Y}$ and the model output $f_{\theta}(\mathbf{X},\mathbf{A})$, which is actually learning information from the mutual information between the labels $\mathbf{Y}$ and the joint distribution $(\mathbf{X}, \mathbf{A})$ because most of GNNs have been demonstrated that satisfy the injective property \cite{GIN} or linear assumption \cite{SGC,S2GC}:
    \begin{equation}\label{GNN goal}
        \max\ I(f_{\theta}(\mathbf{X},\mathbf{A});\mathbf{Y}) \quad \Rightarrow \quad \max\ I((\mathbf{X},\mathbf{A});\mathbf{Y}).
    \end{equation}

    Furthermore, according to the properties of mutual information, we can decompose $I((\mathbf{X},\mathbf{A});\mathbf{Y})$ into $I(\mathbf{X};\mathbf{Y})$ and $I(\mathbf{A};\mathbf{Y}|\mathbf{X})$:
    \begin{equation}\label{MI deco}
        I((\mathbf{X},\mathbf{A});\mathbf{Y}) = I(\mathbf{X};\mathbf{Y})+I(\mathbf{A};\mathbf{Y}|\mathbf{X}).
    \end{equation}
    Thus GNNs' learning objective is actually maximizing $I((\mathbf{X},\mathbf{A});\mathbf{Y})$ and can be decomposed into the maximization of $I(\mathbf{X};\mathbf{Y})$ and the maximization of $I(\mathbf{A};\mathbf{Y}|\mathbf{X})$.

    The goal of the structural attacker is degrading the prediction accuracy of $f_\theta$ as much as possible. To achieve this goal, the structural attacker employs perturbation $\delta^{*}$ to divert the outputs of the victim GNN from the true labels $\mathbf{Y}$ to erroneous predictions $\mathbf{Y}'$, which can be formulated as the minimization of $I((\mathbf{X},(\mathbf{A}+\delta));\mathbf{Y})$ and the maximization of $I((\mathbf{X},(\mathbf{A}+\delta));\mathbf{Y}^{'})$: 
    \begin{equation}
    \mathbf{A}^{'} = \mathbf{A}+\delta^{*}, 
    \end{equation}

    \begin{equation}
    \delta^{*}= \underset{\delta}{\operatorname{arg\,min}}\ I((\mathbf{X},(\mathbf{A}+\delta));\mathbf{Y}).
    \end{equation}

    According to Eq.~\eqref{GNN goal} and Eq.~\eqref{MI deco}, the above goal can be rewritten as:
    \begin{equation}
    \delta^{*}= \underset{\delta}{\operatorname{arg\,min}}\ I(\mathbf{X};\mathbf{Y}) + I((\mathbf{A}+\delta);\mathbf{Y}|\mathbf{X}).
    \end{equation}
    
    Due to $I(\mathbf{X};\mathbf{Y})$ is irrelevant and independent to the structure perturbation $\delta$ thus the actual goal of the attacker is:
    \begin{equation}
    \delta^{*}= \underset{\delta}{\operatorname{arg\,min}}\ I((\mathbf{A}+\delta);\mathbf{Y}|\mathbf{X}).
    \end{equation}
    
    Hence, the structural attacker essentially aims to minimize $I(\mathbf{A}^{'};\mathbf{Y}|\mathbf{X})$ to hinder the victim GNN from extracting adequate information from this mutual information and compel victim GNN to make wrong predictions. We describe this scenario as the contamination of mutual information $I(\mathbf{A}^{'};\mathbf{Y}|\mathbf{X})$.
    
\end{proof}

\subsection{A.2 Proof of Lemma.~2}

\begin{lemma}
    For any $\mathbf{X}^{'}\neq \mathbf{X}$, where $\mathbf{X}^{'},\mathbf{X}\in\mathbb{R}^{n\times d}$, the  mutual information $I(\mathbf{A}^{'}; \mathbf{Y}|\mathbf{X}^{'})$ is less harmful to GNNs than $I(\mathbf{A}^{'}; \mathbf{Y}|\mathbf{X})$.
\end{lemma}

\begin{proof}

According to properties of mutual information, $I(\mathbf{A}^{'}; \mathbf{Y}|\mathbf{X})$ can be reformulated as: 
\begin{equation}\label{condi MI deco}
    I(\mathbf{A}^{'}; \mathbf{Y}|\mathbf{X}) = I(\mathbf{A}^{'}; \mathbf{Y}) - I(\mathbf{A}^{'}; \mathbf{Y}; \mathbf{X}),
\end{equation}

where $I(\mathbf{A}^{'}; \mathbf{Y}; \mathbf{X})$ is the mutual information between $\mathbf{A}^{'}$, $\mathbf{Y}$ and $\mathbf{X}$. 

Lemma.~\ref{lemma:1} indicates attackers degrade victim GNNs' performances by minimizing $I(\mathbf{A}^{'}; \mathbf{Y}|\mathbf{X})$. Eq.~\eqref{condi MI deco} indicates minimizing $I(\mathbf{A}^{'}; \mathbf{Y}|\mathbf{X})$ can be achieved through minimizing $I(\mathbf{A}^{'}; \mathbf{Y})$ and maximizing $I(\mathbf{A}^{'}; \mathbf{Y}; \mathbf{X})$. Assuming the upper bound of $I(\mathbf{A}^{'}; \mathbf{Y}; \mathbf{X})$ is donated as $\tau$, an ideal attacker, in pursuit of minimizing $I(\mathbf{A}^{'}; \mathbf{Y}|\mathbf{X})$, fulfills $I(\mathbf{A}^{'}; \mathbf{Y}; \mathbf{X})=\tau$.

Similarly, for $I(\mathbf{A}^{'}; \mathbf{Y}|\mathbf{X}^{'})$, we have:
\begin{equation}
    I(\mathbf{A}^{'}; \mathbf{Y}|\mathbf{X}^{'}) = I(\mathbf{A}^{'}; \mathbf{Y}) - I(\mathbf{A}^{'}; \mathbf{Y}; \mathbf{X}^{'}).
\end{equation}

Therefore, the difference between $I(\mathbf{A}^{'}; \mathbf{Y}|\mathbf{X}^{'})$ and $I(\mathbf{A}^{'}; \mathbf{Y}|\mathbf{X})$ is referred as:
\begin{align}
    & I(\mathbf{A}^{'}; \mathbf{Y}|\mathbf{X}^{'}) - I(\mathbf{A}^{'}; \mathbf{Y}|\mathbf{X})\nonumber \\ 
    & = I(\mathbf{A}^{'}; \mathbf{Y}) - I(\mathbf{A}^{'}; \mathbf{Y}; \mathbf{X}^{'}) - I(\mathbf{A}^{'}; \mathbf{Y}) + I(\mathbf{A}^{'}; \mathbf{Y}; \mathbf{X})\nonumber \\ 
    & = I(\mathbf{A}^{'}; \mathbf{Y}; \mathbf{X}) - I(\mathbf{A}^{'}; \mathbf{Y}; \mathbf{X}^{'})\nonumber \\ 
    & = \tau - I(\mathbf{A}^{'}; \mathbf{Y}; \mathbf{X}^{'}) \geq 0,
\end{align}

because $I(\mathbf{A}^{'}; \mathbf{Y}; \mathbf{X}^{'})$ is less than the upper bound $\tau$. The above equation demonstrates $I(\mathbf{A}^{'}; \mathbf{Y}|\mathbf{X}^{'}) \geq  I(\mathbf{A}^{'}; \mathbf{Y}|\mathbf{X})$. That implies $I(\mathbf{A}^{'}; \mathbf{Y}|\mathbf{X}^{'})$ retains more information that could potentially be exploited by GNNs. Thus $I(\mathbf{A}^{'}; \mathbf{Y}|\mathbf{X}^{'})$ is less harmful to GNNs.

\end{proof}

\subsection{A.3 Proof of Theorem.~1}
\begin{theorem}
\ourmethod's pre-training stage maximizes the mutual information $I(\mathbf{Z}_{p};\mathbf{Y})$ between $\mathbf{Z}_{p}$ and $\mathbf{Y}$, where $I(\mathbf{A}';\mathbf{Y}|\mathbf{Z}_{p})$ is less harmful to GNNs compared to $I(\mathbf{A}';\mathbf{Y}|\mathbf{X})$.
\end{theorem}

\begin{proof}
$\mathbf{Z}_{p}$ is solely learned from node attributes $\mathbf{X}$, thus 

Minimizing the pre-training loss $\mathcal{L}_{p}$ is actually pursuing the  maximizing 

Minimizing the pre-training loss function $\mathcal{L}_{p}$ aims at ensuring accurate predictions for all nodes, which is maximizing the conditional probability $P(\mathbf{Y}_{v}=c_{v}|\mathbf{Z}_{p(v)})$ for any node $v$. It equals minimizing the information entropy $H(\mathbf{Y}|\mathbf{Z}_{p})$. Due to the properties of mutual information, we have:
\begin{equation}
    I(\mathbf{Y};\mathbf{Z}_{p}) = H(\mathbf{Y}) - H(\mathbf{Y}|\mathbf{Z}_{p}).
\end{equation}

Therefore, minimizing $H(\mathbf{Y}|\mathbf{Z}_{p})$ equals to maximize $I(\mathbf{Y};\mathbf{Z}_{p})$, and naturally, minimizing $\mathcal{L}_{p}$ is equal to maximize $I(\mathbf{Y};\mathbf{Z}_{p})$. Additionally, according to the Lemma.~\ref{lemma:2}, we have:
\begin{equation}
    I(\mathbf{A}';\mathbf{Y}|\mathbf{Z}_{p}) \geq I(\mathbf{A}';\mathbf{Y}|\mathbf{X})=\tau,
\end{equation}

which proves $I(\mathbf{A}';\mathbf{Y}|\mathbf{Z}_{p})$ is less harmful to GNNs compared to $I(\mathbf{A}';\mathbf{Y}|\mathbf{X})$.

\end{proof}

\subsection{A.4 Proof of Lemma.~3}
\begin{lemma}
    There exists an overlap between $I(\mathbf{A}^{'};\mathbf{Y}|\mathbf{Z}_{p})$ and $I(\mathbf{A}^{'};\mathbf{Y}|\mathbf{X})$, which consequently leads to the contamination of $I(\mathbf{A}^{'};\mathbf{Y}|\mathbf{Z}_{p})$.
\end{lemma}

\begin{proof}
We first demonstrate the existence of overlap between $I(\mathbf{A}^{'};\mathbf{Y}|\mathbf{Z}_{p})$ and $I(\mathbf{A}^{'};\mathbf{Y}|\mathbf{X})$. Due to the properties of mutual information, we have:
\begin{align}
    I(\mathbf{A}^{'};\mathbf{Y}|\mathbf{Z}_{p}) \subset I(\mathbf{A}';\mathbf{Y}), \quad I(\mathbf{A};\mathbf{Y}|\mathbf{Z}_{p}) \subset I(\mathbf{A}';\mathbf{Y}) \nonumber \\
    I(\mathbf{A}^{'};\mathbf{Y}|\mathbf{X})=I(\mathbf{A}^{'};\mathbf{Y}) - I(\mathbf{A}^{'};\mathbf{Y};\mathbf{X}),
\end{align}

thus $I(\mathbf{A}^{'};\mathbf{Y}|\mathbf{Z}_{p})\cup I(\mathbf{A}^{'};\mathbf{Y}|\mathbf{X})=0$ if and only if $I(\mathbf{A}^{'};\mathbf{Y}|\mathbf{Z}_{p})=I(\mathbf{A}^{'};\mathbf{Y};\mathbf{X})$, which is not always feasible to guarantee in practice. Consequently, we have:
\begin{equation}
I(\mathbf{A}^{'};\mathbf{Y}|\mathbf{Z}_{p})\cup I(\mathbf{A}^{'};\mathbf{Y}|\mathbf{X})\neq 0.
\end{equation}

To demonstrate the possible contamination of $I(\mathbf{A}^{'};\mathbf{Y}|\mathbf{Z}_{p})$, we abstract the $f_{\theta}$ as a Simplified Graph Convolution (SGC), where $f_{\theta}$ is a linearized function with parameter $W_{\theta}$:
\begin{equation}
    f_{\theta}(\mathbf{X}) = \mathbf{X}\cdot W_{\theta}.
\end{equation}

In structure fine-tuning, the SGC model with $K$ layers convolution can be formulated as:
\begin{equation}
    f_{\theta_{p}}(\mathbf{X}) = (\mathbf{A}^{'})^{k}\cdot\mathbf{X}\cdot W_{\theta_{p}}, \quad \mathbf{Z}_{p} = \mathbf{X}\cdot W_{\theta_{p}}.
\end{equation}

Suppose the parameter update during the fine-tuning process is denoted as $\Delta W$, the output of the function $f_{\theta_{p}}$ after fine-tuning is:
\begin{align}
    f_{\theta_{p}}(\mathbf{X}) &= (\mathbf{A}^{'})^{k}\cdot\mathbf{X}\cdot (W_{\theta_{p}}+\Delta W) \nonumber \\
    &=(\mathbf{A}^{'})^{k}\cdot\mathbf{X}\cdot W_{\theta_{p}} +  (\mathbf{A}^{'})^{k}\cdot\mathbf{X}\cdot \Delta W \nonumber \\
    &=(\mathbf{A}^{'})^{k}\cdot \mathbf{Z}_{p} + (\mathbf{A}^{'})^{k}\cdot\mathbf{X}\cdot \Delta W.
\end{align}

The term following the addition can be regarded as the embedding output by a GNN with parameters $\Delta W$, taking $\mathbf{X}$ and $\mathbf{A}^{'}$ as inputs. It is actually equal to a victim GNN, whose embedding is contaminated by $I(\mathbf{A}^{'};\mathbf{Y}|\mathbf{X})$. That implies that during the fine-tuning stage, $\mathbf{Z}_{p}$ unavoidably incorporates the influence of $I(\mathbf{A}^{'};\mathbf{Y}|\mathbf{X})$.

\end{proof}

\subsection{A.5 Proof of Theorem.~2}
\begin{theorem}
    \ourmethod's structure fine-tuning stage maximizes $I(\mathbf{A}^{'};\mathbf{Y}| \mathbf{Z}_{p}^{'})$ to learn structural information and align it to $I(\mathbf{A}^{'};\mathbf{Y}| \mathbf{X_{inter}})$ to prevent from being contaminated.
\end{theorem}

\begin{proof}
Following the principles of mutual information, we hold:
\begin{align}
    I((\mathbf{X_{inter}},\mathbf{A}^{'});\mathbf{Y}) &= I(\mathbf{A}^{'};\mathbf{Y}) + I(\mathbf{X_{inter}};\mathbf{Y}|\mathbf{A}^{'})\nonumber \\
    &\leq I(\mathbf{A}^{'};\mathbf{Y}) + I(\mathbf{X_{inter}};\mathbf{Y})\nonumber \\
    &\leq I(\mathbf{A}^{'};\mathbf{Y}) +  H(\mathbf{Y}) - H(\mathbf{Y}|\mathbf{X_{inter}}),
\end{align}

where $H(\cdot)$ is information entropy. $H(\mathbf{Y})$ is the information entropy of labels which is solely determined by $\mathbf{Y}$ and independent of $\mathbf{X}$. As for the conditional entropy $H(\mathbf{Y}|\mathbf{X_{inter}})$ which measures the uncertainty of $\mathbf{Y}$ given the $\mathbf{X_{inter}}$. An effective $\mathbf{X_{inter}}$ should be able to predict $\mathbf{Y}$ well, meaning that knowing $\mathbf{X_{inter}}$ allows us to determine the value of $\mathbf{Y}$ with great certainty. However, InterNAA intentionally replaces the node features contained in $\mathbf{X_{inter}}$ with features from nodes of different classes, leading to an inability to accurately predict $\mathbf{Y}$ through $\mathbf{X_{inter}}$, leading to a large $H(\mathbf{Y}|\mathbf{X_{inter}})$. Due to the non-negative characteristic of mutual information:
\begin{equation}
    I(\mathbf{X_{inter}};\mathbf{Y}) = H(\mathbf{Y}) - H(\mathbf{Y}|\mathbf{X_{inter}})\geq 0.
\end{equation}

When the conditional entropy is sufficiently large, the mutual information $I(\mathbf{X_{inter}};\mathbf{Y})$ tends to be 0, at which point the above equation fulfills:

\begin{align}
    I((\mathbf{X_{inter}},\mathbf{A}^{'});\mathbf{Y})\Rightarrow I(\mathbf{A}^{'};\mathbf{Y}).
\end{align}

which means $I(f(\mathbf{X_{inter}},\mathbf{A}^{'});\mathbf{Y}))$ is approximately equal to $I(\mathbf{A}^{'};\mathbf{Y}))$, and align $\mathbf{Z}_{p}$ with $f(\mathbf{X_{inter}},\mathbf{A}^{'})$ is actually let $\mathbf{Z}_{p}$ combine structure information from $I(\mathbf{A}^{'};\mathbf{Y})$ instead of the contaminated $I(\mathbf{A}^{'};\mathbf{Y}|\mathbf{X})$.
\end{proof}
\begin{table*}[t]
\centering
\begin{tabular}{cccccc}
\hline
\multicolumn{1}{c}{\textbf{Datasets}} & \textbf{Hom. Ratio} & \textbf{Nodes} & \textbf{Edges} & \textbf{Features} & \textbf{Classes}        \\ \hline
Cora                                  & 0.81                & 2,708          & 5,429          & 1,433             & 7                       \\
Cora-ml                               & 0.80                & 2,810          & 7,981          & 2,879             &7                        \\
Citeseer                              & 0.74                & 3,327          & 4,732          & 3,703             & 6                       \\
Pubmed                                & 0.80                & 19,717         & 44,338         & 500               & 3                       \\
Chameleon                             & 0.23                & 2,277          & 36,101         & 2,325             & 4                       \\
Squirrel                              & 0.22                & 5,201          & 198,353        & 2,089             & 5                       \\
ogbn-Arxiv & - & 169,343         & 1,157,799         & 128               & 40                       \\
ogbn-Products & - & 2,449,029         & 61,859,076         & 100               & 47                       \\
\hline
\end{tabular}
\caption{Dataset Statistics.}
\label{tab:datasets}
\end{table*}

\begin{table*}[t]
\centering
\begin{tabular}{c|c|c|c|c|c|c|c|c}
\hline
Dataset & \multicolumn{4}{c|}{Chameleon}& \multicolumn{4}{c}{Squirrel} \\ \hline
Ptb & 0 & 5 & 10 & 20 & 0 & 5 & 10 & 20 \\ \hline
GCN & 56.3±1.6 & 52.1±2.0 & 49.2±2.2 & 40.6±1.9 & 41.2±0.6 & 38.4±0.5 & 36.8±0.7 & 34.3±0.3 \\
RGCN & 54.7±1.5 & 53.8±1.4 & 51.0±1.7 & 41.3±1.8 & 40.5±0.2 & 38.6±0.4 & 34.3±0.6 & 32.9±0.2 \\
GCN-Jaccard & - & - & - & - & - & - & - & - \\
Pro-GNN & 56.1±0.4 & 54.8±0.9 & 50.2±1.3 & 48.1±1.0 & \textbf{42.0±0.1} & \underline{39.7±0.5} & 37.6±0.9 & 36.1±1.0 \\
SimP-GCN & 55.6±1.4 & 54.0±0.9 & 50.5±1.7 & 46.4±1.8 & 40.9±0.2 & 38.3±0.6 & 35.1±0.5 & 32.3±0.7 \\
EvenNet & \textbf{57.3±1.2} & \underline{55.6±1.9} & \textbf{52.5±2.0} & \textbf{49.0±2.2} & 41.3±0.9 & \textbf{40.0±1.1} & \textbf{38.8±0.7} & \underline{37.0±1.3} \\
STABLE & 54.0±1.3 & 52.2±1.9 & 49.1±2.6 & 39.9±2.8 & \underline{41.4±0.7} & 37.7±0.8 & 35.6±0.9 & 31.5±0.7 \\
GADC & 56.8±2.2 & 54.7±1.9 & 51.0±2.4 & \underline{48.8±1.6} & 41.1±0.9 & 38.9±1.1 & 37.7±0.5 & 36.9±1.1 \\
Noisy-GCN & 56.5±1.3 & 53.0±1.1 & 50.1±1.6 & 46.2±1.8 & 40.0±0.7 & 37.3±0.8 & 36.1±0.9 & 34.7±1.3 \\
\hline
\ourmethod\ & \underline{57.0±3.7} & \textbf{55.9±2.5} & \underline{52.0±1.8} & \underline{48.8±1.7} & 40.9±1.0 & \underline{39.7±1.3} & \underline{38.1±0.9} & \textbf{37.4±0.9} \\ \hline
\end{tabular}
\caption{Robustness Comparison Mettack. 
}
\label{tab:metattack_appendix}
\end{table*}

\section{B. Experimental Details}

\subsection{B.1 Dataset Details}
\label{appendix:dataset}
In this section, we describe in detail the graph datasets used in this paper. We report statistics for these datasets in (Table.~\ref{tab:datasets}). Each dataset is described below:

Cora, Citeseer, and Pubmed are three benchmark datasets commonly used in graph neural network research, each representing different scenarios in scientific literature citation networks. In these networks, nodes represent papers; edges indicate citations of that paper by other papers, and node labels are the academic topics of the papers. Among them, Cora and Citeseer datasets use bag-of-words 0/1 vectors to represent terms present in the paper as node features, while PubMed uses TF-IDF vectors as its node features.

Chameleon and Squirrel are heterophilic graph datasets commonly used in GNN. Chameleon is a web page link network where the nodes represent web pages, and the edges represent hyperlinks between them. On the other hand, Squirrel is a Wikipedia article network, with nodes denoting the articles and edges indicating the hyperlinks between them. In contrast to the previously mentioned datasets, the node features in Chameleon and Squirrel are not based on text content but rather describe the properties of the nodes (web pages and Wikipedia articles).

\subsection{B.2 Baseline Descriptions}
\label{appendix:baseline}

\begin{itemize}
    \item \textbf{GCN}: It is the most representative GNN model which utilizes the graph convolutional layer to propagate node features with the low-pass filter and smooth the features of connected node pairs.
    \item \textbf{RGCN}: It learns the Gaussian distributions for each node feature and employs an attention mechanism to alleviate the potential malicious influence of nodes with high variance.
    \item \textbf{GCN-Jaccard}: It sanitizes the graph data by pruning links that connect nodes with low values of Jaccard similarity of node attributes.
    \item \textbf{ProGNN}: It jointly learns a structural graph and a robust GNN model from the modified graph guided by the three properties: low-rank, sparsity and feature smoothness.
    \item \textbf{SimP-GCN}: It utilizes a kNN graph to capture the node similarity and enhance the node representation of the GNN.
    \item \textbf{STABLE}: It utilizes the homophily assumption to refine the modified structure and combine contrastive learning techniques to remove adversarial edges. Finally, an advanced GCN is used to predict node labels.
    \item \textbf{EvenNet}: By applying balance theory, it obtains a more robust spectral graph ﬁlter under homophily change by ignoring messages from odd-order neighbors and only using even-order terms.
    \item \textbf{GADC}: Inspired by graph diffusion convolution, it proposes a novel min-max optimization to perturb graph structure based on Laplacian distance. 
    \item \textbf{NoisyGCN}: It injects random noise into the hidden states of the GCN to improve the robustness against structure attacks.
\end{itemize}

\subsection{B.3 Configuration for Baselines and Datasets}
\label{appendix:configuration}
The attack ratios of the two methods are set to {5\%, 10\%, 15\%, 20\%, 25\%}. Our baseline methods include representatives of vanilla GNNs, such as GCN \cite{GCN}, as well as typical adaptive aggregation methods like RGCN \cite{RGCN} and SimP-GCN \cite{SimP-GCN}, purifying methods including GCN-SVD \cite{GCN-SVD}, GCN-Jaccard \cite{GCN-Jaccard}, Pro-GNN \cite{Pro-GNN}, and the most recent method, HANG-quad \cite{hang-quad}. Detailed information about baselines can be found in the Appendix.~\ref{appendix:baseline}.

We follow the data splitting method of DeepRobust: randomly selecting 10\% of the nodes for training, 10\% for validation, and the remaining 80\% for testing. In accordance with the victim GNN configuration of DeepRobust, \ourmethod is configured with 2 layers, 16 hidden units and a dropout ratio of 0.5, while the learning rate was set to 0.01. We search the training epoch of \ourmethod\ over {200,300,400} and fine-tuning epoch over {3, 10, 20, 50}. For PubMed, to mitigate the computational overhead associated with InterNAA sampling, we random select 20\% of nodes in the training set to replace their attributes with those from nodes belonging to different classes. 

\section{C. Defense Performance in Heterophilic Graphs}
\label{appendix:hetero}

Table.~\ref{tab:metattack_appendix} demonstrates the performance of the proposed model and baselines on two typical heterophilic graph datasets. We report the average classification accuracy (± standard deviation) of 10 runs with different perturbation ratios (ptb). The best and second-best results are highlighted in bold and underlined, respectively. Due to errors reported by GCN-Jaccard on these two datasets, its results are represented by "-". The proposed model achieves the best or second-best performance in most cases, demonstrating the robustness of \ourmethod\ on heterophilic graphs.

\bibliography{aaai25}

\end{document}